\documentclass[journal]{IEEEtran}

\ifCLASSINFOpdf
  \usepackage{graphicx}
\else
   \usepackage[dvips]{graphicx}
\fi
\usepackage{url}

\usepackage{amsmath}
\usepackage{amssymb}
\usepackage{mathptmx}
\usepackage{amsfonts}
\usepackage{amsthm}
\newtheorem{proposition}{Proposition}
\usepackage{mathtools}
\usepackage[utf8]{inputenc}
\newcommand{\E}{\mathbb{E}}
\newcommand{\R}{\mathbb{R}}
\newcommand{\bx}{\mathbf{x}}
\newcommand{\bs}{\mathbf{s}}
\newcommand{\br}{\mathbf{r}}
\newcommand{\bz}{\mathbf{z}}
\newcommand{\by}{\mathbf{y}}
\newcommand{\bW}{\mathbf{W}}
\newcommand{\bM}{\mathbf{M}}
\newcommand{\bA}{\mathbf{A}}
\newcommand{\bI}{\mathbf{I}}
\newcommand{\bC}{\mathbf{C}}
\newcommand{\bb}{\mathbf{b}}
\DeclareMathOperator{\tr}{tr}

\begin{document}

\title{Reservoir Subspace Injection for Online ICA under Top-$n$ Whitening}

\author{Wenjun Xiao, Yuda Bi, Vince Calhoun, \IEEEmembership{Fellow, IEEE}
\thanks{W. Xiao is with the Department of Computer Science, The George Washington University, Washington, DC 20052 USA.}
\thanks{Y. Bi is with the Tri-Institutional Center for Translational Research in Neuroimaging and Data Science (TReNDS), Georgia State University, Georgia Institute of Technology, and Emory University, Atlanta, GA 30303 USA (e-mail: ybi@gsu.edu).}
\thanks{V. D. Calhoun is with the Tri-Institutional Center for Translational Research in Neuroimaging and Data Science (TReNDS), Georgia State University, Georgia Institute of Technology, and Emory University, Atlanta, GA 30303 USA, and also with the School of Electrical and Computer Engineering, Georgia Institute of Technology, Atlanta, GA 30332 USA (e-mail: vcalhoun@gsu.edu).}}

\markboth{IEEE Signal Processing Letters}
{Xiao \MakeLowercase{\textit{et al.}}: Reservoir Subspace Injection for Online ICA under Top-$n$ Whitening}
\maketitle

\begin{abstract}
Reservoir expansion can improve online independent component analysis (ICA)
under nonlinear mixing, yet top-$n$ whitening may discard injected features.
We formalize this bottleneck as \emph{reservoir subspace injection} (RSI):
injected features help only if they enter the retained eigenspace without
displacing passthrough directions. RSI diagnostics (IER, SSO, $\rho_x$)
identify a failure mode in our top-$n$ setting: stronger injection increases
IER but crowds out passthrough energy ($\rho_x: 1.00\!\rightarrow\!0.77$),
degrading SI-SDR by up to $2.2$\,dB. A guarded RSI controller preserves
passthrough retention and recovers mean performance to within $0.1$\,dB of
baseline $1/N$ scaling. With passthrough preserved, RE-OICA improves over
vanilla online ICA by $+1.7$\,dB under nonlinear mixing and achieves positive
SI-SDR$_{\mathrm{sc}}$ on the tested super-Gaussian benchmark ($+0.6$\,dB).
\end{abstract}


\section{Introduction}

Independent component analysis (ICA) recovers latent independent sources
from their linear mixtures and, under the standard linear instantaneous
mixing model, is identifiable whenever at most one source is
Gaussian~\cite{Comon1994,Hyvarinen2001}. Classical batch algorithms---Infomax~\cite{bell1995},
FastICA~\cite{Hyvarinen1999}, JADE~\cite{Cardoso1993JADE}---have been deployed
extensively in neuroimaging~\cite{Calhoun2001,Calhoun2012Review},
telecommunications, and audio processing; see~\cite{HyvarinenOja2000,Cardoso1998}
for surveys and textbook treatment~\cite{Hyvarinen2001}. However, batch ICA
requires full data and a fixed mixing matrix, motivating sample-wise online ICA.
Amari's natural gradient~\cite{Amari1998,AmariCichocki1996} provides an
equivariant update with reduced sensitivity to mixing conditioning, and online recursive ICA (ORICA)
combines recursive-least-squares whitening with natural-gradient updates for
real-time EEG source separation~\cite{Hsu2014ORICA}. Our vanilla online ICA is
ORICA-inspired, not a full reimplementation: it keeps ORICA's
natural-gradient demixing update~\cite{Hsu2014ORICA} but replaces RLS whitening
with EMA covariance and periodic eigendecomposition (every 64 steps). It is the
reservoir-free RE-OICA ablation and is intended to isolate RSI effects rather
than provide a head-to-head ORICA benchmark.
Stochastic-approximation methods further extend online ICA to
tensorial or multivariate-component settings~\cite{LiJordan2021,kim2006independent,anderson2014independent}. Nevertheless, these approaches operate
in the observation space: under nonlinear mixing or nonlinear source
distortions~\cite{TalebJutten1999PNL,BachJordan2002KernelICA}, a purely
\emph{linear} demixing in $\bx$ can be insufficient.
Nonlinear ICA typically requires auxiliary variables or temporal structure for
identifiability~\cite{hyvarinen2023nonlinear,HyvarinenMorioka2016TCL,Khemakhem2020VAEICA}, while deep-learning solutions
often sacrifice online, low-latency processing.

Reservoir computing~\cite{Jaeger2001ESN,Jaeger2004Science} offers an
alternative to observation-space demixing: an echo-state network (ESN) maps
inputs into fixed high-dimensional nonlinear features, approximating
fading-memory functionals~\cite{Lukosevicius2012Practical}.
Recent results relate ESNs to kernel and random-feature views of nonlinear
dynamics~\cite{Dong2020RecurrentKernel,ESNStateSpace2025}. Reservoir-based
separation has been explored for offline denoising with a trained
readout~\cite{Choi2024ReservoirBSS}; in contrast, reservoir-expanded online ICA (RE-OICA) targets fully online,
sample-by-sample blind source separation (BSS) with no reservoir/readout training.
A practical issue is that low-latency online pipelines typically apply top-$n$
whitening, which can discard injected reservoir features; without retention
diagnostics, it is difficult to determine when expansion is beneficial.

This gap motivates the present work. We recast reservoir-expanded online ICA
as a \emph{subspace-injection} problem and contribute:
\begin{itemize}
\item[(i)] RSI diagnostics (IER, SSO, $\rho_x$) that quantify injected-feature
retention and passthrough preservation under top-$n$ whitening.
\item[(ii)] A crowd-out mechanism---validated in controlled experiments
(10 seeds, SI-SDR$_{\mathrm{sc}}$)---showing that stronger injection can raise
IER yet reduce $\rho_x$ and degrade SI-SDR; improvements under nonlinear mixing
occur when passthrough is preserved.
\item[(iii)] A crowd-out-guarded controller that adapts $\alpha_t$ to maintain
high $\rho_x$ while allowing beneficial injection, with negligible overhead
beyond whitening eigendecomposition.
\end{itemize}

\section{Model and Main Results}

\subsection{Setup and Notation}

We observe $\bx_t\in\R^n$ generated from $n$ independent sources
$\bs_t\in\R^n$ with $\E[\bs_t]=\mathbf{0}$, $\E[\bs_t\bs_t^\top]=\bI_n$, and
finite fourth moments. We consider:
\begin{align}
\text{(Static)}\;& \bx_t = \bA\,\bs_t, \label{eq:static}\\
\text{(Time-varying)}\;& \bx_t = \bA(t)\,\bs_t,\quad
\bA(t)=\bA_0+\varepsilon\sin\!\bigl(\tfrac{2\pi f t}{T}\bigr)\boldsymbol{\Delta},
\label{eq:tv}\\
\text{(Nonlinear)}\;& \bx_t = g(\bA\,\bs_t)+\boldsymbol{\eta}_t,
\label{eq:nl}
\end{align}
where $\bA,\bA_0,\boldsymbol{\Delta}\in\R^{n\times n}$, $g$ is component-wise
Lipschitz (e.g., $\tanh$), $\boldsymbol{\eta}_t$ is i.i.d.\ Gaussian, $f$ is
the modulation frequency, and $T$ is the sequence horizon used for
normalization.

RE-OICA uses a fixed reservoir and projection, concatenates passthrough and
injected features, then applies top-$n$ whitening and natural-gradient ICA:
\begin{align}
\br_t &= \mathrm{ESN}(\bx_{1:t})\in\R^N, \nonumber\\
\mathbf{p}_t &= c_N\bW_{\mathrm{read}}\br_t\in\R^d, \qquad
\mathbf{u}_t = [\bx_t;\,\alpha_t\mathbf{p}_t]\in\R^{n+d}, \nonumber\\
\bz_t &= \mathrm{Whiten}_n(\mathbf{u}_t)\in\R^n, \qquad
\by_t = \mathrm{ICA}(\bz_t)\in\R^n, \label{eq:pipeline}
\end{align}
where $\alpha_t$ is the RSI-controlled injection scale (defined below).
Unless stated otherwise, $n{=}3$, $N{=}500$, $d{=}20$. We use top-$k$
whitening with $k{=}n$ throughout.

\subsection{Reservoir Encoding}

We use a sparse echo-state network (ESN) reservoir~\cite{Jaeger2001ESN,Yildiz2012ESP,GrigoryevaOrtega2018ESNUniversal} with
fixed random weights. Let $\bW_{\mathrm{in}}\in\R^{N\times n}$,
$\bW_{\mathrm{res}}\in\R^{N\times N}$ (sparse, spectral radius $\rho$), and
$\bb\in\R^N$. The state update is
\begin{equation}\label{eq:reservoir}
\br_t=(1-\alpha_r)\br_{t-1}+\alpha_r\,\sigma(\bW_{\mathrm{in}}\bx_t+\bW_{\mathrm{res}}\br_{t-1}+\bb),
\end{equation}
with leak rate $\alpha_r\in(0,1]$ and $\sigma=\tanh$ applied element-wise.

\begin{proposition}[Echo State Property]\label{prop:esp}
Let $\sigma$ be $L_\sigma$-Lipschitz. If
$\rho_{\mathrm{eff}}:=(1-\alpha_r)+\alpha_r L_\sigma\|\bW_{\mathrm{res}}\|<1$,
then for any two initial states $\br_0,\br_0'$ driven by the same input,
$\|\br_t-\br_t'\|\le\rho_{\mathrm{eff}}^{\,t}\|\br_0-\br_0'\|$.
\end{proposition}
\begin{proof}
Let $\boldsymbol a_t=\bW_{\mathrm{in}}\bx_t+\bb$. By Lipschitzness,
\begin{align*}
\|\br_t-\br_t'\|
&\le (1-\alpha_r)\|\br_{t-1}-\br_{t-1}'\|
 +\alpha_r L_\sigma\|\bW_{\mathrm{res}}\|\,\|\br_{t-1}-\br_{t-1}'\| \\
&= \rho_{\mathrm{eff}}\|\br_{t-1}-\br_{t-1}'\|.
\end{align*}
Iterating yields the claim.
\end{proof}
Using the induced operator norm in $\|\bW_{\mathrm{res}}\|$ gives a sufficient
(possibly conservative) contraction condition.
Thus the dependence on the initial state decays exponentially (fading memory).

\textbf{Readout projection and passthrough.}
With fixed random $\bW_{\mathrm{read}}\in\R^{d\times N}$, define
$\mathbf{p}_t=c_N\bW_{\mathrm{read}}\br_t\in\R^d$ and
$\mathbf{u}_t=[\bx_t;\alpha_t\mathbf{p}_t]\in\R^{n+d}$.
This reduces whitening cost from $O(N^3)$ to $O((n{+}d)^3)$ while preserving a
raw-input floor. We test $c_N\in\{1/N,1/\sqrt{N}\}$; effective injection
magnitude is $\alpha_t c_N$, with $c_N$ fixed per branch and $\alpha_t$
adapted only at whitening refreshes.

\subsection{Motivation: Random-Feature Approximation}

Random-feature expansion~\cite{RahimiRecht2007,vershynin2018high} motivates
reservoir lifting in nonlinear demixing, but the operational bottleneck here is
retention after top-$n$ whitening.

\begin{proposition}[Motivational linearization bound]\label{prop:linearization}
Under standard random-feature integral-representation assumptions for
$\psi=\bA^{-1}\!\circ g^{-1}$, the population-optimal linear readout from
$N$ memoryless random features obeys, where
$\Phi:\R^n\!\to\!\R^N$ is the random-feature map and $B$ bounds the
integral-representation coefficient norm:
\begin{equation}\label{eq:linearization}
\E_{\bW_{\mathrm{in}},\bb}\!\Bigl[\min_{\bM}\E_{\bx}\bigl[\|\bM\Phi(\bx)-\bs\|^2\bigr]\Bigr] \;\le\; \frac{n\,B^2}{N}.
\end{equation}
\end{proposition}

\emph{Remark.}
Proposition~\ref{prop:linearization} is an in-principle population result for
memoryless features. Our core contribution is top-$n$ retention diagnostics and
control in the online recurrent pipeline.

\subsection{Reservoir Subspace Injection (RSI)}

We form the whitening input
\begin{equation}
\mathbf{u}_t=[\bx_t;\alpha_t\mathbf{p}_t],\qquad
\mathbf{p}_t=c_N\bW_{\mathrm{read}}\br_t,
\end{equation}
and let $\mathbf{V}_n\in\R^{(n+d)\times n}$ be the top-$n$ eigenvectors of
$\bC_t(\alpha_t)$ at each whitening refresh. Define the unscaled concatenation
$\tilde{\mathbf{u}}_t=[\bx_t;\mathbf{p}_t]$ with block covariance
\[
\tilde{\bC}_t=\mathrm{Cov}(\tilde{\mathbf{u}}_t)=
\begin{bmatrix}\bC_{xx}&\bC_{xp}\\\bC_{px}&\bC_{pp}\end{bmatrix}.
\]
Under injection, $\mathbf{u}_t=[\bx_t;\alpha_t\mathbf{p}_t]$ has covariance
\[
\bC_t(\alpha_t)=\begin{bmatrix}\bC_{xx}&\alpha_t\bC_{xp}\\
\alpha_t\bC_{px}&\alpha_t^2\bC_{pp}\end{bmatrix}.
\]
Let $\mathbf{P}_x=\mathrm{diag}(\bI_n,\mathbf{0}_{d\times d})$ and
$\mathbf{P}_p=\mathrm{diag}(\mathbf{0}_{n\times n},\bI_d)$ be coordinate
projectors. Reservoir features help only if they are retained in the top-$n$
eigenspace of $\bC_t(\alpha_t)$, quantified by
\begin{align}
E_p &= \tr\!\left(\mathbf{V}_n^\top \mathbf{P}_p \bC_t(\alpha_t)\mathbf{P}_p \mathbf{V}_n\right),\\
E_x &= \tr\!\left(\mathbf{V}_n^\top \mathbf{P}_x \bC_t(\alpha_t)\mathbf{P}_x \mathbf{V}_n\right),\\
\mathrm{IER}_t &= \frac{E_p}{E_x+E_p},\qquad
\mathrm{SSO}_t = \frac{\|\mathbf{P}_p\mathbf{V}_n\|_F^2}{\|\mathbf{V}_n\|_F^2}.
\end{align}
IER is the retained-energy share from reservoir coordinates, whereas SSO is the
retained subspace-overlap fraction on reservoir coordinates (scale-insensitive).
Since
$\mathbf{V}_n^\top\mathbf{V}_n=\bI_n$, we have $\|\mathbf{V}_n\|_F^2=n$ and
$\mathrm{SSO}_t\in[0,1]$. We also track the \emph{retained passthrough ratio}
$\rho_{x,t}=E_x/\tr(\bC_{xx})$, i.e., passthrough variance retained after
top-$n$ selection.

\textbf{Crowd-out-guarded controller.}
We pose injection control as $\max_{\alpha}\mathrm{IER}(\alpha)$ subject to
$\rho_x(\alpha)\ge\rho_x^\star$, and implement a tracking rule:
$\alpha_t$ is driven up when $\mathrm{IER}_t<\mathrm{IER}^\star$ and penalized
when $\rho_{x,t}<\rho_x^\star$:
\begin{equation}\label{eq:rsi_ctrl}
\delta_t=\kappa(\mathrm{IER}^\star-\mathrm{IER}_t)
-\kappa_g[\rho_x^\star-\rho_{x,t}]_+/\rho_x^\star,
\end{equation}
\begin{equation}
\alpha_{t+1}=\mathrm{clip}\!\bigl(\alpha_t e^{\delta_t},\alpha_{\min},\alpha_{\max}\bigr),
\end{equation}
where $[\cdot]_+=\max(\cdot,0)$, $\kappa>0$ is the IER tracking gain,
$\kappa_g>0$ is the guard gain, and $[\alpha_{\min},\alpha_{\max}]$ are clipping
bounds. We set $\rho_x^\star$ conservatively (e.g., $0.95$) to retain ${\ge}95\%$
passthrough variance; performance is insensitive to the exact value as long as
operation avoids the crowd-out regime ($\rho_x<0.9$; Table~\ref{tab:mechanism}).
Similarly, $\mathrm{IER}^\star$ is set as a practical operating target for
stable online operation under the stated protocol. The update uses already-computed
covariance/eigenspace statistics.

\begin{proposition}[Reservoir-entry condition (block-diagonal case)]\label{prop:rsi}
Let $\tilde{\bC}_t$ be block-partitioned as
$\begin{bmatrix}\bC_{xx}&\bC_{xp}\\\bC_{px}&\bC_{pp}\end{bmatrix}$ under
$\tilde{\mathbf{u}}_t=[\bx_t;\mathbf{p}_t]$, and assume $\bC_{xp}=\mathbf{0}$.
If
\begin{equation}
\lambda_{\max}(\alpha_t^2\bC_{pp})>\lambda_n(\bC_{xx}),
\end{equation}
where $\lambda_n(\bC_{xx})$ denotes the $n$-th largest eigenvalue of $\bC_{xx}$,
then at least one top-$n$ eigenvector has nonzero reservoir component, so
$\mathrm{SSO}_t>0$.
\end{proposition}
\begin{proof}[Sketch]
With $\bC_{xp}=\mathbf{0}$ in $\tilde{\bC}_t$,
$\bC_t(\alpha_t)=\mathrm{diag}(\bC_{xx},\alpha_t^2\bC_{pp})$,
so its spectrum is the union of the spectra of the two blocks. The stated
strict inequality implies at least one reservoir-block eigenvalue enters the
top-$n$ set, so the retained eigenspace has nonzero reservoir overlap.
\end{proof}

\emph{Observation (crowd-out in the block-diagonal case).}
Under $\bC_{xp}=\mathbf{0}$ in $\tilde{\bC}_t$, if
$\lambda_{\max}(\alpha_t^2\bC_{pp})>\lambda_n(\bC_{xx})$, at least one
reservoir-dominated direction enters the top-$n$ eigenspace and replaces a
passthrough direction, implying a decrease in $\rho_{x,t}$.
For $\bC_{xp}\neq\mathbf{0}$ in $\tilde{\bC}_t$, coupling in
$\bC_t(\alpha_t)$ perturbs this ordering.
With a non-negligible eigengap around $\lambda_n(\bC_{xx})$ and moderate
$\|\bC_{xp}\|$, Davis--Kahan-type bounds imply continuous variation from the
block-diagonal limit; crowd-out can still arise when reservoir-induced
variance becomes competitive in the top-$n$ spectrum.

\emph{Remark.}
The block-diagonal limit makes the mechanism explicit: top-$n$ selection is a
competition between passthrough and reservoir block spectra. In experiments,
normalized cross-block coherence is approximately constant across injection
strengths, so performance changes are driven mainly by relative block energies
rather than large coupling shifts.

\textbf{Interpretation.}
In the block-diagonal limit, SSO${>}0$ (Proposition~\ref{prop:rsi}) is necessary for reservoir
influence but not sufficient for improved separation: when reservoir
eigenvalues displace passthrough eigenvalues from the top-$n$ basis,
$\rho_x$ drops and SI-SDR degrades (\emph{crowd-out};
see Table~\ref{tab:mechanism} for quantification).
The guarded controller~\eqref{eq:rsi_ctrl} mitigates this by
adapting $\alpha_t$ to maintain $\rho_{x,t}$ near $\rho_x^\star$ in practice.

\subsection{Online Whitening and Natural Gradient ICA}

\textbf{Online whitening.}
For $\mathbf{u}_t=[\bx_t;\alpha_t\mathbf{p}_t]\in\R^{n+d}$, we update the mean
and covariance with forgetting factor $\lambda\in(0,1)$:
\begin{align}
\boldsymbol{\mu}_t &= \lambda\,\boldsymbol{\mu}_{t-1} + (1-\lambda)\,\mathbf{u}_t, \label{eq:mu}\\
\bC_t &= \lambda\,\bC_{t-1} + (1-\lambda)\,(\mathbf{u}_t-\boldsymbol{\mu}_t)(\mathbf{u}_t-\boldsymbol{\mu}_t)^\top. \label{eq:cov}
\end{align}
Before eigendecomposition we apply diagonal loading
$\bC_t\leftarrow\bC_t+\epsilon\bI$ ($\epsilon=10^{-6}$). Every $\Delta_w$
steps, we compute the top-$n$ eigenspace $\mathbf{V}_n$ and eigenvalues
$\mathbf{D}_n$, form $\bW_{\mathrm{wh}}=\mathbf{D}_n^{-1/2}\mathbf{V}_n^\top$,
and whiten
$\bz_t=\bW_{\mathrm{wh}}(\mathbf{u}_t-\boldsymbol{\mu}_t)\in\R^n$.
At each refresh we evaluate IER/SSO and $\rho_{x,t}$ using the current
top-$n$ eigenspace and update $\alpha_{t+1}$ via~\eqref{eq:rsi_ctrl};
$\alpha_t$ is held fixed between refreshes. Retaining top-$n$ eigenvectors
matches the source count and removes low-variance directions; RSI controls
whether reservoir directions remain in this retained subspace.

\textbf{Natural gradient ICA.}
On $\bz_t$, we update the demixing matrix $\bW_t\in\R^{n\times n}$ via the
natural gradient~\cite{Amari1998}:
\begin{equation}\label{eq:natgrad}
\bW_{t+1}=\bW_t+\eta\bigl(\bI_n-\phi(\by_t)\by_t^\top\bigr)\bW_t,\qquad
\by_t=\bW_t\bz_t,
\end{equation}
where $\phi=\tanh$ and $\eta>0$. Every $\Delta_o$ steps we apply
$\bW\leftarrow(\bW\bW^\top)^{-1/2}\bW$ to prevent drift.

\textbf{Algorithm 1 (RSI-controlled RE-OICA).}
Per sample: update $\br_t,\mathbf{p}_t$; form $\mathbf{u}_t=[\bx_t;\alpha_t\mathbf{p}_t]$;
on refresh recompute top-$n$ whitening, evaluate IER/SSO/$\rho_{x,t}$, and update
$\alpha_{t+1}$ via~\eqref{eq:rsi_ctrl}; whiten to $\bz_t$; apply~\eqref{eq:natgrad}.
\begin{figure*}[t]
  \centering
  \includegraphics[width=0.95\textwidth]{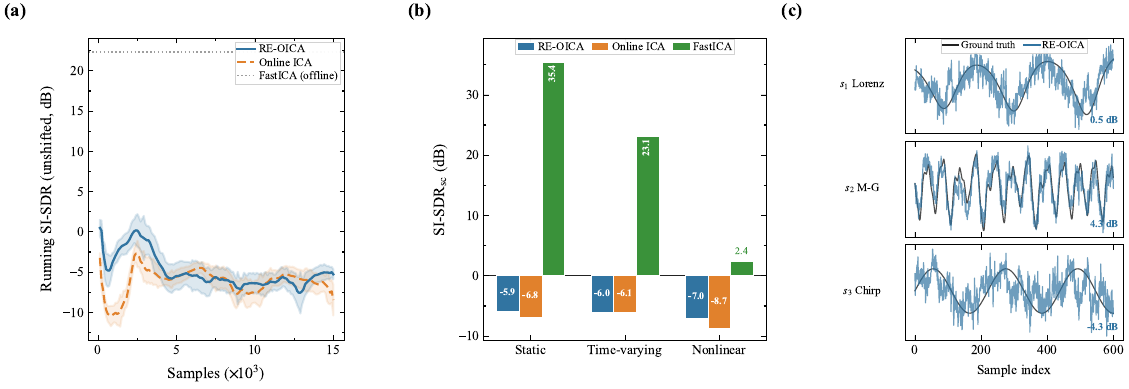}
  \caption{RE-OICA validation ($T{=}15{,}000$, $n{=}3$, 10 seeds). (a) Running unshifted SI-SDR (trailing 2\,000): RE-OICA, vanilla, FastICA. (b) Steady-state SI-SDR$_{\mathrm{sc}}$ across regimes (lag-compensated). (c) Nonlinear overlay (last 600): true vs.\ RE-OICA; corner labels show per-source unshifted SI-SDR.}
  \label{fig:fig1}
\end{figure*}
\section{Experiments}

\textbf{Sources and mixing.}
We use $n{=}3$ standardized sources: Lorenz ($x$-component),
Mackey--Glass ($\tau{=}17$), and a linear chirp (0.5--5\,Hz over 10\,s).
Mixing follows~\eqref{eq:static}--\eqref{eq:nl}: static ($\bA$ fixed),
time-varying ($\varepsilon{=}0.3$, $f{=}0.5$\,Hz), and nonlinear
($g{=}\tanh(\gamma\cdot)$, $\gamma{=}0.8$, SNR$=10$\,dB with additive white Gaussian noise (AWGN)).
Mixing matrices are random with condition numbers $\le 5$.

\textbf{Methods.}
We compare RE-OICA-base (ESN reservoir: $N{=}500$, sparsity $5\%$, $\rho{=}0.95$,
$\alpha_r{=}0.1$, $c_{\mathrm{in}}{=}0.1$ input scaling, fixed random
$\bW_{\mathrm{read}}\in\R^{d\times N}$ with $d{=}20$, baseline $c_N{=}1/N$),
RE-OICA-sqrt ($c_N{=}1/\sqrt{N}$), RE-OICA-RSI (guarded~\eqref{eq:rsi_ctrl}
with $\mathrm{IER}^\star{=}0.25$, $\rho_x^\star{=}0.95$, $\kappa{=}0.3$,
$\kappa_g{=}3$, $\alpha_t\in[0.1,10]$), an unguarded RSI counterfactual
($\mathrm{IER}^\star{=}0.25$, no $\rho_x$ guard), and baselines vanilla online
ICA and FastICA.
All online methods share the same backend: EMA whitening ($\lambda{=}0.995$,
$\Delta_w{=}64$) and natural-gradient ICA ($\eta{=}5{\times}10^{-3}$,
$\Delta_o{=}50$). \emph{Vanilla online ICA} (ORICA-inspired) applies this
backend directly to $\bx_t$ without reservoir expansion (same
natural-gradient update as ORICA~\cite{Hsu2014ORICA}, but EMA plus periodic
eigendecomposition whitening rather than ORICA's RLS whitening).
\emph{FastICA}~\cite{Hyvarinen1999} is the offline reference
(full batch, \texttt{logcosh} nonlinearity; cf.~\cite{Hyvarinen2001,Miettinen2017SquaredFastICA}).
Online runs use a 1000-sample whitening warm-up (ICA frozen) and 2000-sample
learning-rate ramp; Fig.~\ref{fig:fig1}(a) reports 10-seed
mean$\pm$standard error of the mean (SEM) running unshifted SI-SDR.

\textbf{Metrics and alignment.}
On the last 5{,}000 samples, we compute lag-aware correlations
$\rho_{ij}=\max_{|\tau|\le 200}|\mathrm{corr}(s_i, y_j^{(\tau)})|$ and use Hungarian matching on $[\rho_{ij}]$ to obtain a permutation $\pi$; matched outputs are evaluated at their maximizing lags on the overlap.
We report SI-SDR$_{\mathrm{sc}}$~\cite{LeRoux2019SISDR,Vincent2006BSSEval}, mean $|r|$, and RSI diagnostics (IER/SSO and $\rho_x$).
Running curves show trailing 2{,}000-sample \emph{unshifted} SI-SDR, whereas steady-state SI-SDR$_{\mathrm{sc}}$ is lag-compensated.
Results are 10-seed mean$\pm$SEM; counts are descriptive.

\begin{table}[t]
\centering
\caption{Last 5\,000 samples (10 seeds, mean$\pm$SEM): SI-SDR$_{\mathrm{sc}}$ (dB), mean $|r|$, count (RE-OICA $>$ vanilla).}
\label{tab:results}
\setlength{\tabcolsep}{3pt}
\footnotesize
\begin{tabular}{l l r@{$\,\pm\,$}l r@{$\,\pm\,$}l c}
\hline
Regime & Method & \multicolumn{2}{c}{SI-SDR$_{\mathrm{sc}}$}
 & \multicolumn{2}{c}{Corr} & Count \\
\hline
Static      & RE-OICA &$-5.9$&$1.0$ &$0.44$&$0.04$ & 6/10 \\
            & Vanilla &$-6.8$&$1.0$ &$0.41$&$0.04$ &     \\
\hline
Time-var.   & RE-OICA &$-6.0$&$1.1$
             &$0.44$&$0.04$ & 7/10 \\
            & Vanilla &$-6.1$&$1.1$ &$0.44$&$0.04$ &     \\
\hline
Nonlinear   & RE-OICA &$\mathbf{-7.0}$&$\mathbf{1.1}$
             &$\mathbf{0.41}$&$\mathbf{0.04}$ & 7/10 \\
            & Vanilla &$-8.7$&$1.0$ &$0.35$&$0.03$ &     \\
\hline
\end{tabular}
\vspace{2pt}

{\scriptsize FastICA (batch): static $35.4/1.00$, time-var.\ $23.1/1.00$,
nonlinear $2.4/0.78$ (SI-SDR$_{\mathrm{sc}}$/Corr).}
\end{table}

\textbf{Results (Fig.~\ref{fig:fig1}, Table~\ref{tab:results}).}
\emph{(a) Convergence.} Under time-varying mixing, Fig.~\ref{fig:fig1}(a)
shows 10-seed mean$\pm$SEM running unshifted SI-SDR; both methods stabilize
after ${\sim}3{,}000$ samples.

\emph{(b) Cross-regime comparison (Table~\ref{tab:results}).}
Under this protocol, RE-OICA outperforms vanilla in all three regimes, with
the largest gain under nonlinear mixing ($+1.7$\,dB; Table~\ref{tab:results}).
FastICA is an offline reference in this benchmark; online
SI-SDR$_{\mathrm{sc}}$ is negative on these chaotic/oscillatory stress-test
sources, where even batch FastICA reaches only $+2.4$\,dB under nonlinear mixing.
Win counts are descriptive rather than inferential.

\emph{Standard BSS validation.}
On standard super-Gaussian sources (Laplace, square wave, sawtooth;
$T{=}30{,}000$; same mixing-matrix protocol with condition number ${\le}5$ and
same whitening/ICA hyperparameters), RE-OICA-base ($1/N$ scaling) achieves
$+0.6{\pm}1.2$\,dB vs.\ $-2.9{\pm}2.3$\,dB for vanilla
($\Delta{=}{+}3.5$\,dB, 7/10 seeds).
Under mild nonlinearity ($g{=}\tanh(0.5\cdot)$, SNR$=20$\,dB), RE-OICA-base
achieves $-1.5{\pm}1.0$\,dB vs.\ $-3.6{\pm}1.4$\,dB
($\Delta{=}{+}2.1$\,dB, 9/10 seeds).
Negative chaotic-source values are therefore consistent with task difficulty
under this protocol.

\emph{(c) Waveform quality.} Overlay (nonlinear, last 600 samples) shows
RE-OICA tracks Lorenz and chirp dynamics qualitatively; Mackey--Glass remains
the most difficult source (Fig.~\ref{fig:fig1}(c)).

\emph{Computational cost.}
Per-sample cost is dominated by reservoir encoding (sparse
$\bW_{\mathrm{res}}$ multiply, $O(\mathrm{nnz})$ with
$\mathrm{nnz}=\mathrm{nnz}(\bW_{\mathrm{res}})$); the
$O((d{+}n)^3/\Delta_w)$ amortized eigendecomposition and $O(n^2)$ ICA
update are modest for $d{+}n{=}23$, $n{=}3$ in our implementation setting.

\textbf{Ablation studies (Fig.~\ref{fig:fig2}, Table~\ref{tab:mechanism}).}
All ablations follow the same protocol (10 seeds) and SI-SDR$_{\mathrm{sc}}$
evaluation.

\emph{(a) $N$-sweep and RSI mechanism (Fig.~\ref{fig:fig2}(a),
Table~\ref{tab:mechanism}).}
Under baseline $1/N$ scaling, SI-SDR$_{\mathrm{sc}}$ is nearly flat across $N$
with very low IER and $\rho_x{\approx}1$.
Unconstrained $1/\sqrt{N}$ scaling increases IER but degrades SI-SDR$_{\mathrm{sc}}$
via crowd-out, while the guarded RSI branch recovers near-baseline performance
(Table~\ref{tab:mechanism}).

\emph{(b) Architecture (Fig.~\ref{fig:fig2}(b)).}
Recurrent ESN and memoryless random features (RF) perform similarly and are
slightly above vanilla under time-varying mixing ($\Delta{=}{+}0.1$--$0.2$\,dB),
suggesting that, in this benchmark, high-dimensional nonlinear expansion is
likely the primary contributor to the gain rather than fading memory.

\emph{(c) Drift strength (Fig.~\ref{fig:fig2}(c)).}
Across $\varepsilon\in\{0.1,0.3,0.8\}$, RE-OICA remains comparable to or slightly
above vanilla ($\Delta{=}{+}0.1$--$0.8$\,dB) with no monotonic dependence on
$\varepsilon$.

Additional ablations on readout dimension, passthrough removal, and nonlinearity
type showed qualitatively consistent trends (results not shown due to space
constraints).

\begin{table}[t]
\centering
\caption{Nonlinear (10 seeds): SI-SDR$_{\mathrm{sc}}$ (dB, mean$\pm$SEM), IER/$\rho_x$ steady-state, count vs.\ RE-OICA-base.}
\label{tab:mechanism}
\setlength{\tabcolsep}{3pt}
\scriptsize
\begin{tabular}{l r@{$\,\pm\,$}l c c c}
\hline
Branch & \multicolumn{2}{c}{SI-SDR$_{\mathrm{sc}}$} & IER & $\rho_x$ & Count \\
\hline
RE-OICA-base ($1/N$)     &$-7.0$&$1.1$ & ${<}0.001$ & $1.00$ & (ref) \\
RE-OICA-sqrt ($1/\!\sqrt{N}$) &$-8.7$&$1.0$ & $0.016$ & $0.94$ & 2/10 \\
RSI-unguarded (counterfactual, $\mathrm{IER}^\star\!{=}0.25$) &$-9.2$&$1.0$ & $0.255$ & $0.77$ & 2/10 \\
RE-OICA-RSI (guarded, $\rho_x^\star\!{=}0.95$)  &$-7.1$&$1.1$ & $0.007$ & $0.98$ & 5/10 \\
\hline
\end{tabular}
\end{table}

\textbf{Mechanism summary (Table~\ref{tab:mechanism}).}
RSI diagnostics reveal a crowd-out tradeoff: increasing IER can reduce passthrough retention $\rho_x$ and degrade SI-SDR$_{\mathrm{sc}}$, while the guarded controller preserves $\rho_x$ and recovers near-baseline performance. Cross-block coherence remains stable ($0.38$--$0.45$), suggesting increased reservoir-block energy $\|\bC_{pp}\|$ rather than coupling changes.

\begin{figure}[t]
  \centering
  \includegraphics[width=\columnwidth]{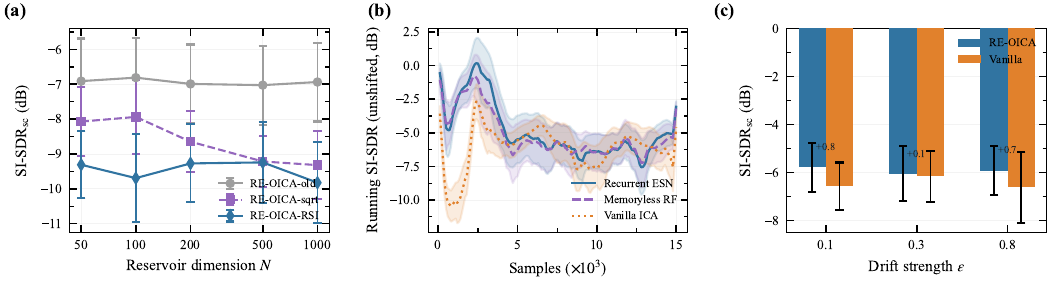}
  \caption{Ablation studies (10 seeds, mean$\pm$SEM).
  \textbf{(a)} $N$-sweep under nonlinear mixing across three scaling branches.
  \textbf{(b)} Architecture: ESN vs.\ RF vs.\ vanilla (time-varying).
  \textbf{(c)} Drift sweep: $\varepsilon\in\{0.1,0.3,0.8\}$.}
  \label{fig:fig2}
\end{figure}

\section{Conclusion}

Top-$n$ whitening makes reservoir expansion a rank-budget problem: features
outside the retained eigenspace cannot help. RSI diagnostics (IER, SSO, $\rho_x$)
predict outcome; stronger injection can raise IER yet reduce $\rho_x$ and SI-SDR
(Table~\ref{tab:mechanism}). A low-overhead guarded controller maintains
$\rho_x\!\approx\!\rho_x^\star$ and recovers mean performance to within $0.1$\,dB
of baseline $1/N$ scaling under our protocol. Future work should increase
\emph{useful} cross-block structure rather than merely $\|\bC_{pp}\|$.

\bibliographystyle{IEEEtran}
\bibliography{ref}

\end{document}